%% file: bib.tex
\documentclass{article}

\usepackage{microtype}
\usepackage{graphicx}
\usepackage{subfigure}
\usepackage{booktabs} 

\usepackage{hyperref}

\usepackage{url}
\usepackage{amsmath}
\usepackage{relsize}

\usepackage{amsfonts}
\usepackage{multirow}
\usepackage{array}
\usepackage{color}
\usepackage[colorinlistoftodos,prependcaption,textsize=tiny]{todonotes}
\usepackage{amsthm}

\newtheorem*{nonumlemma}{Lemma}

\usepackage{xcolor} 
\usepackage{xspace}
\usepackage{comment}
\usepackage[noabbrev,capitalize]{cleveref}

\usepackage{enumitem}
\usepackage[olditem,oldenum]{paralist}

\usepackage[accepted]{icml2019}

\input{macros}

\newtheorem{lemma}{Lemma}

\icmltitlerunning{Probabilistic Neural-symbolic Models for VQA}

\begin{document}

\twocolumn[
\icmltitle{Probabilistic Neural-symbolic Models for \\Interpretable
Visual Question Answering}



\icmlsetsymbol{partgatech}{*}

\begin{icmlauthorlist}
\icmlauthor{Ramakrishna Vedantam}{partgatech,fair}
\icmlauthor{Karan Desai}{gatech}
\icmlauthor{Stefan Lee}{gatech}
\icmlauthor{Marcus Rohrbach}{fair}
\icmlauthor{Dhruv Batra}{fair,gatech}
\icmlauthor{Devi Parikh}{fair,gatech}
\end{icmlauthorlist}

\icmlaffiliation{fair}{Facebook AI Research}
\icmlaffiliation{gatech}{Georgia Tech}

\icmlcorrespondingauthor{Ramakrishna Vedantam}{ramav@fb.com}

\icmlkeywords{Machine Learning, ICML}

\vskip 0.3in
]



\printAffiliationsAndNotice{\textsuperscript{*} Part of this work was done when R.V. was at Georgia Tech.}  


\begin{abstract}
    \ramapass{We propose a new class of probabilistic neural-symbolic models,
    that have symbolic functional programs as a latent, stochastic variable.
    Instantiated in the context of visual question answering, our probabilistic
    formulation offers two key conceptual advantages over prior neural-symbolic models for VQA.
    Firstly, the programs generated by our model are more understandable while requiring
    less number of teaching examples. Secondly, we show that one can pose counterfactual
    scenarios to the model, to probe its beliefs on the programs that could
    lead to a specified answer given an image. Our results on the \ramaarxiv{CLEVR and SHAPES 
    datasets} verify our hypotheses, showing that the model gets better
    program (and answer) prediction accuracy even in the low data regime, and
    allows one to probe the coherence and consistency of reasoning performed.}
\end{abstract}

\section{Introduction}
\input{sections/intro.tex}

\section{Methods}
\input{sections/methods_theorem.tex}

\section{Related Work}\label{sec:related_work}
\input{sections/related.tex}
\section{Experimental Setup}\label{sec:experiments}
\input{sections/experiments.tex}

\section{Results}\label{sec:results}
\input{sections/results.tex}

\section{Discussion and Conclusion}
\input{sections/conclusion.tex}

\bibliography{bib}
\bibliographystyle{icml2019}

\clearpage
\twocolumn[\icmltitle{Appendix}]
\input{appendix/appendix.tex}

\end{document}

%% file: macros.tex
\newcommand{\ramapass}{\textcolor{black}}
\newcommand{\ramaarxiv}{\textcolor{black}}
\newcommand{\ramacr}{\textcolor{black}}
\newcommand{\retwo}{\textcolor{black}}
\newcommand{\reb}{\textcolor{black}}

\newcommand{\nsym}{neural-symbolic\xspace}

\newcommand{\myvec}[1]{\mathbf{#1}}

\newcommand{\va}{\myvec{a}}

\newcommand{\vi}{\myvec{i}}

\newcommand{\vx}{\myvec{x}}

\newcommand{\vz}{\myvec{z}}

\newcommand{\pmvar}[1]{\scriptsize{$\pm$#1}}

\newcommand{\KLpq}[2]{\mathbb{KL}\left({#1}||{#2}\right)}

\newcommand{\KL}{\mathrm{KL}}

\newcommand{\etc}{\emph{etc.}}
\newcommand{\ie}{\emph{i.e.}\xspace}
\newcommand{\eg}{\emph{e.g.}\xspace}

\newcommand{\vs}{\emph{vs.}\xspace}

\newcommand{\eat}[1]{}
\newcommand{\remark}[2]{} 
\newcommand{\replace}[2]{} 

\newcommand{\stefan}[1]{}
\newcommand{\elbo}{\mathrm{elbo}}

\newcommand{\vqabayes}{Prob-NMN}



%% file: sections/intro.tex
\label{sec:intro}
Building flexible learning and reasoning machines is a central challenge in 
Artificial Intelligence (AI). Deep representation learning~\cite{LeCun2015}
provides us powerful, flexible function approximations that
have resulted in state-of-the-art
performance across multiple AI tasks such as
recognition~\citep{Krizhevsky2012,He2015msra}, 
machine translation~\citep{Sutskever2014}, visual question
answering~\citep{Antol2015}, speech modeling~\cite{Van_den_Oord2016-un},
and reinforcement learning~\citep{Mnih2015DQN}.
However, many
aspects of human cognition such as systematic compositional
generalization
\ramaarxiv{(\eg, understanding that ``John loves Mary'' could
imply that ``Mary loves John'')}~\cite{Lake2017,Lake2017Gen}
have proved harder to model.

Symbol manipulation~\citep{Newell1976-gv},
on the other hand lacks flexible
learning capabilities but supports
strong generalization and systematicity~\cite{Lake2017Gen}.
Consequently, many works have focused
on building \nsym{} models with the aim of combining the best of 
representation learning and symbolic reasoning
~\citep{Bader_2005,Evans_2018,Valiant_2003,Yi2018,Yin2018}.

\begin{figure*}
	\centering
	\includegraphics[width=\textwidth]{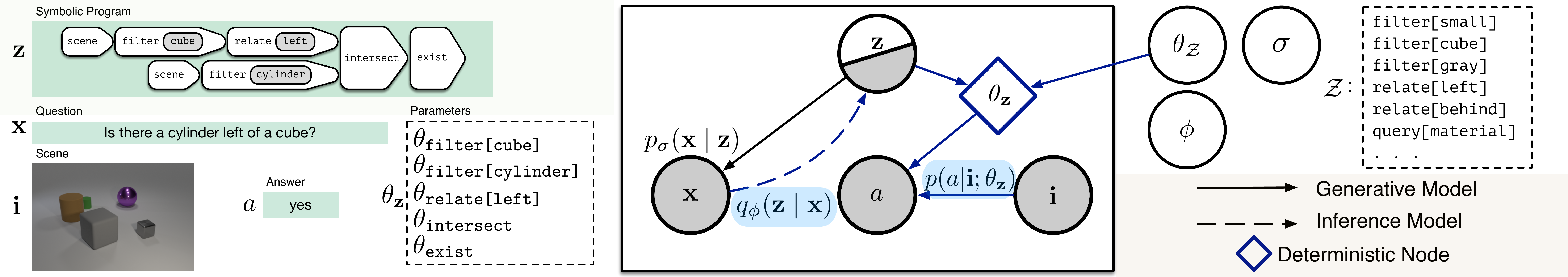}
	\caption{\vspace{-2pt}\textbf{Probabilistic Neural Symbolic Models for VQA:} We show our novel
	probabilistic
	\nsym{} model (\textbf{\vqabayes{}}) in plate notation (right). Given image $\vi$, programs
	$\vz$ are a (partially observed) latent variable executed on the image to generate an answer $a$
	using parameters $\theta_{\vz}$ (left). For VQA, we infer $\vz$ given question $\vx$
	to generate an answer. Inference on $\vz$ given $a$ and $\vi$ tests coherence
	of reasoning patterns. Baseline non-probabilsitic neural symbolic models (\textbf{NMNs})
	capture (subset of) terms and edges shown in blue, and are less data-efficient
	and less interpretable reasoning models than our (probabilistic) proposal.\vspace{-10pt}}
	\label{fig:model}
\end{figure*}

As we scale machine learning to machine reasoning
~\cite{Andreas2016-vj,Andreas2016NMN,Bottou2011-bb,Weston2015-sd}
a natural desire is to provide guidance to the model in the form of
instructions. In such a context, symbols
are more intuitive to specify than say the parameters of a neural network.
Thus, a promising method for interpretable reasoning models 
is to specify the reasoning plan 
symbolically and learn to execute it using deep learning.

This~\nsym{} methodology has been extensively used to model
reasoning capabilities in visual
question answering (VQA)~\citep{Andreas2016NMN,Hu2017NMN,Mao2018,Mascharka2018,Johnson2017,
Yi2018} and
to some extent in reinforcement learning
~\citep{Andreas2016-vj,Das2018}. 
Concretely, in the VQA task one is given an image $\vi$,
a question $\vx$ (``\emph{Is there a cylinder left of a cube?}''),
for which we would like to provide an answer $a$ (\texttt{yes}).
In addition, one may also
optionally be provided a program $\vz$ for the question that specifies
a reasoning plan. For \eg, one might ask a model to
apply the \texttt{filter[cube]} operator,
follwed by \texttt{relate[left]}, and then \texttt{And}
the result together with \texttt{filter[cylinder]} to
predict the answer, where each operator is parameterized as a neural
network (\cref{fig:model}).

The scope of this current work is to provide a probabilistic
framework for such~\nsym{} models. This provides two benefits:
firstly, given a limited number of teaching
examples of plans / programs for a given question / context,
we show one can better capture the association between
the question and programs to provide more understandable and legible
program explanations for novel questions. Inspired by~\citet{Dragan2013},
we call this notion data-efficient legibility, since the model's program (actions)
in this case need to be legible,~\ie clearly convey the
question (goal specification) the model has been given.

Secondly, the formulation makes it possible to probe deeper
into the model's reasoning capabilities.
For \eg, one can test if the reasoning done by the system is:
1) coherent: programs which lead to similar answers are consistent
with each other and 2) sensitive: tweaking the
answer meaningfully changes the underlying reasoning plan.


\ramacr{Our probabilistic model treats 
functional programs ($\vz$) as a stochastic latent variable.}
This allows us to share statistics
meaningfully between questions with associated programs
and those without corresponding
programs. This aids data-efficient legibility.
Secondly, probabilistic modeling brings to bear a rich set of 
inferential techniques to answer conditional queries on the beliefs of the model.
After fitting a model for VQA, one can sample $\vz \sim
p(\vz| \vi, a)$, to probe if the reasoning done by the model is
coherent (\ie multiple programs leading to answer \texttt{yes} are coherent)
and sensitive to a different answer ($a$) (say, \texttt{no}).

Given an image ($\vi$), we build a model for 
$p(\vx, \vz, a| \vi)$ (see~\cref{fig:model}); where the model factorizes as:
$p(\vx, \vz, a| \vi) = p(\vz) p(\vx| \vz) p(a| \vz, \vi)$.
The implied generative process is as follows. First we sample a program $\vz$,
which generates questions $\vx$. Further, given a program $\vz$ and an image $\vi$
we generate answers $a$. Note that based on the symbolic program $\vz$, we dynamically
instantiate parameters of a neural network $\theta_\vz$ (these are deterministic,
given $\vz$), by composing smaller neural-modules
for each symbol in the program. This is similar to prior work on neural-symbolic VQA
~\cite{Hu2017NMN,Johnson2017} using neural module networks (NMNs).
In comparison to prior works, our probabilistic formulation (shorthand~\vqabayes) leads to better semi-supervised learning, and reasoning capabilities\footnote{Note that this model assumes independence of programs
from images, which corresponds to the weak sampling assumptions
in concept learning~\citep{Tenenbaum1999-ps}, one can handle question
premise, \ie that people might ask a specific set of questions
for an image in such a model by reparameterizing the answer variable to include a
relevance label.}.


Our technical contribution is to formulate semi-supervised learning
with this deep generative neural-symbolic model using variational inference~\cite{Jordan1999VI}.
First, we derive variational lower bounds on the evidence
for the model for the semi-supervised and supervised cases, and show how this motivates
the semi-supervised objectives used in previous work with discrete structured latent spaces
~\cite{Miao2016,Yin2018}. Next, we show how to learn program execution,~\ie $p(a|\vz, \vi)$
jointly with the remaining terms in the model by devising a stage-wise optimization
algorithm \ramacr{inspired by~\cite{Johnson2017}.}


%
\textbf{Contributions.}
First, we provide tractable algorithms for training models with probabilistic latent programs
	that also learn to execute them in an end-to-end manner.
	Second, we take the first steps towards deriving useful semi-supervised learning
	objectives for the class of structured sequential latent variable models~\cite{Miao2016,Yin2018}.
	Third, our approach enables interpretable reasoning systems that are more 
	legible with less supervision, and expose their reasoning process to test 
	coherence and sensitivity.
Fourth, our system offers improvements on the \ramaarxiv{CLEVR~\citep{Johnson2017}
	as well as SHAPES~\citep{Andreas2016NMN} datasets in the low question program supervision regime.
	That is, our model answers questions more accurately and with legible (understandable) programs
	than adaptations of prior work to our setting as well as baseline non-probabilistic
	variants of our model.}

%% file: sections/methods_theorem.tex
\label{sec:vqabayes}
\ramapass{We first explain the model and its parameterization, 
then detail the training objectives along with a discussion of
a stage-wise training procedure.}

Let $\vi \in \mathbb{R}^{U\times V}$ be an input image, 
$\vx$ be a question, which is comprised of a sequence of words
$(x_1, \cdots, x_t)$, where each $x_t \in \mathcal{X}$, where
$\mathcal{X}$ is the vocabulary comprising all words,
similarly, $a \in \mathcal{A}$ be the answer
(where $\mathcal{A}$ is the answer vocabulary),
and $\vz$ be the prefix serialization of a program.
That is, $\vz = (z_1, \cdots, z_T)\ni z_t \in
\mathcal{Z}$, where $\mathcal{Z}$ is the program token
vocabulary.

The model we describe below assumes $\vz$
is a latent variable that is observed only for a subset of datapoints in $\mathcal{D}$.
\ramapass{Concretely, the programs $\vz$ express (prefix-serializations of)
tree structured graphs. Computations are performed given this
symbolic representation by instantiating, for each symbol
$z \in \mathcal{Z}$ a corresponding neural network (\cref{fig:model}, right) with parameters
$\theta_z$ (\cref{fig:model}). That is, given a symbol in the program,
say \texttt{find[green]}, the model instantiates parameters $\theta_{\texttt{find[green]}}$.
In this manner, the mapping $p(a| \vi, \vz)$ is operationalized as $p(a| \vi; \theta_\vz)$,
where $\theta_\vz = \{\theta_{z_t}\}_{t=1}^T$.}


\ramacr{Our model factorizes as
$p(\vx, \vz, a| \vi) = p(\vz) p(\vx|\vz) p(a|\vi; \theta_\vz)$
(see~\cref{fig:model}), where parameters $\theta_{\vz}$
are a deterministic node in the graph instantiated given the program
$\vz$ and $\theta_\mathcal{Z}$ (which
represents the concatenation of parameters across all tokens in $\mathcal{Z}$).
In addition
to the generative model, we also use an inference network $q_{\phi}(\vz|\vx)$
to map questions to latent structured programs. Thus, the generative parameters
in the model are $\{\theta_{\mathcal{Z}}, \sigma\}$, and inference parameters are ${\phi}$.}

\subsection{Parameterization} \ramacr{The terms $p(\vz)$, $p_\sigma(\vx| \vz)$ and $q_{\phi}(\vz|\vx)$ are all
parameterized as LSTM neural networks~\citep{Hochreiter1997-xj}. The program prior $p(\vz)$ is pretrained using
maximum likelihood on programs
simulated using the syntax of the program or a held-out dataset of programs. The
prior is learned first and kept fixed for the rest of the training process.
Finally, the parameters $\theta_z$ for symbols $z$ parameterize small, deep
convolutional neural networks which optionally take as input an attention map over the image
(see Appendix for more details).}

For reference, previous non-probabilistic, neural-symbolic models for VQA~\cite{Hu2017NMN,Johnson2017,Mascharka2018}
model $q_\phi(\vz | \vx)$ and $p(\va | \vi;\theta_{\vz})$ terms from our present model (see blue arrows in~\cref{fig:model}).

\subsection{Learning}
\ramapass{We assume access to a dataset $\mathcal{D} = \{\vx^n, \vz^n\} \cup \{\vx^m, a^m, \vi^m\}$ 
where $m$ indexes the visual question answering dataset, with questions, corresponding images
and answers, while $n$ indexes a teaching dataset, which provides the corresponding programs 
for a question, explaining the steps required to solve the question.}
For \ramaarxiv{data-efficient} legibility, we are interested in the setting where
$N < M$~\ie we have few annotated examples of programs which might be
more expensive to specify.

Given this, learning in our model consists of estimating parameters
$\{\theta_{\mathcal{Z}}, \sigma, \phi\}$. We do this in a
stage-wise fashion (shown below) (c.f.~\citet{Johnson2017}):

\textbf{Stage-wise Optimization.}\\
1) Question Coding: Optimizing parameters $\{\sigma, \phi\}$ to learn a good code for
    questions $\vx$ in the latent space $\vz$.\\
    2) Module Training: Optimizing parameters $\theta_{\mathcal{Z}}$ for learning
    to execute symbols $z$ using neural networks $\theta_z$.\\
    3) Joint Training: Learning all the parameters of the model $\{\theta_{\mathcal{Z}}, \sigma, \phi\}$.\\
%
\ramapass{We describe each of the stages in detail below, and defer the reader to the Appendix for a
more details.}

\noindent\textbf{Question Coding.}
\ramacr{We fit the model on the evidence for questions $\vx$
and programs $\vz$, marginalizing answers $a$,~\ie{} $\sum_n \log p(\vx^n, \vz^n)
+ \sum_m \log p(\vx^m)$. We lower-bound the second
term, optimizing the evidence lower bound (ELBO) with an amortized inference
network $q_\phi$ (c.f.~\citet{Kingma2013-wx,Rezende2014})}:
\vspace{-5pt}
\begin{align}\label{eqn:qc_elbo}
    \begin{split}
\sum_m \log p(\vx^m) &\ge \sum_m \mathbb{E}_{\vz\sim q_\phi(\vz|\vx^m)}[\log p_\sigma(\vx^m| \vz) - \\
&\beta \log q_\phi(\vz| \vx^m) + \beta \log p(\vz)] = \mathcal{U}^\beta_{qc}
    \end{split}
\end{align}
where the lower bound holds for $\beta > 1$. \ramacr{In practice,
we follow prior work in violating the bound, using
$\beta < 1$ to scale the contribution from
$\KLpq{q(\vz|\vx)}{p(\vz)}$
in~\cref{eqn:qc_elbo}. This has been shown to be important for learning meaningful representations
when decoding sequences (c.f.~\citet{Alemi2017broken,Bowman2016,Miao2016,Yin2018}).
For more context, \citet{Alemi2017broken}
provides theoretical justifications for why this is desirable when decoding
sequences which we discuss in more detail in the Appendix.}

\ramaarxiv{The $\mathcal{U}_{qc}$ term in~\cref{eqn:qc_elbo} does not capture the semantics of programs,
in terms of how they relate to particular questions. For modeling legible programs,}
we would also like to make use of the labelled data $\{\vx^n, \vz^n\}$
to learn associations between questions and programs, and provide legible explanations
for novel questions $\vx$. To do this, one can factorize the model to maximize
$\mathcal{L} = \sum_n \log p(\vx^n| \vz^n) + \log p(\vz^n)$. \ramapass{While in theory
given the joint, it is possible to estimate $p(\vz| \vx)$, this is expensive and requires
an (in general) intractable sum over all possible programs. Ideally, one would like to
reuse
the \emph{same} variational approximation $q_\phi(\cdot)$
that we are training for $\mathcal{U}_{qc}$ so that it
learns from both labelled as well
as unlabelled data (c.f.~\citet{Kingma2014SemiSupervisedLW}).
We prove the following lemma and use it to construct an objective that makes
use of $q_\phi$, and relate it to the evidence.}

\begin{lemma}\label{lemma:question_code}
    Given observations $\{\vx^n, \vz^n\}$ and $p(\vx, \vz) = p(\vz) p_{\sigma}(\vx| \vz)$,
    let $z_t$, the token at the $t^{th}$ timestep in a sequence $\vz$ be distributed
    as a categorical with parameters $\pi_t$. Let us denote 
    $\Pi = \{\pi_t\}_{t=1}^{T}$, the joint random variable over all $\pi_t$. Then,
    the following is a lower bound on the joint evidence $\mathcal{L} = \sum_n \log p(\vz^n) + \log p_\sigma(\vx^n| \vz^n)$:
    \vspace{-10pt}
    \begin{multline}
    \mathcal{L} \ge \sum_{n=1}^N \log q_{\phi}(\vz^n|\vx^n) + \log p_{\sigma}(\vx^n|\vz^n) -\\
    \KLpq{q(\Pi|\vz^n, \vx^n)}{p(\Pi)}
    \end{multline}\label{eqn:question_code_sup}
    where $p(\Pi)$ is a distribution over the sampling distributions implied by the prior $p(\vz)$
    and $q(\Pi|\vx^n, \vz^n) = q(\pi_0) \prod_{t=1}^T q(\pi_t|\vx^n, z_0^n, \cdots, z_{t-1}^n)$, where
    each $q(\pi_t|\cdot)$ is a delta distribution on the probability simplex.
\end{lemma}
See Appendix for a proof of the result. \ramacr{This is an extension
of the result for a related graphical model (with discrete $\vz$ observations)
from~\cite{Kingma2014SemiSupervisedLW,Keng2017-lp} to
the case where $\vz$ is a sequence.}

\ramaarxiv{In practice, }the bound above \ramacr{not useful,} as the proof assumes a delta posterior
$q(\Pi|\vz, \vx)$ which makes the last $\KL$ term $\infty$. This means we have to resort to
learning only the first two terms in~\cref{eqn:question_code_sup} as an approximation:
\begin{equation}\label{eqn:approx_sup_evidence}\vspace{-5pt}
\mathcal{L} \approx \sum_{n=1}^N \alpha \log q_{\phi}(\vz^n|\vx^n) + \log p_{\sigma}(\vx^n|\vz^n)
\end{equation}
where $\alpha > 1$ is a scaling on $\log q_\phi$, also used in
~\citet{Kingma2013-wx,Miao2016}
\ramacr{(see Appendix for more details and a justification).}


\noindent{\textbf{Connections to other objectives.}} To our knowledge, two previous works
~\citep{Miao2016,Yin2018}
have formulated semi-supervised learning with discrete (sequential) latent variable
models. While~\citet{Miao2016} write the supervised term as $\log p_{\phi}(\vz| \vx)$,
~\citet{Yin2018} write it as $\log q_{\phi}(\vz| \vx) + \log p_{\sigma}(\vx| \vz)$.
The lemma above provides a clarifying perspective on both the objectives, firstly
showing that $p_{\phi}$ should actually be written as $q_{\phi}$ (and
suggests an additional $p_{\sigma}$ term), and that
the objective from~\citet{Yin2018} is actually a part of a loose lower bound
on the evidence for $\mathcal{L}$ providing some justification for the intuition
presented in~\citet{Yin2018}\footnote{A promising direction for obtaining tighter bounds could be to change the parameterization
of the variational $q(\Pi)$ distribution. \ramapass{Overall, learning of 
$q_\phi$ is challenging in the structured, discrete space of sequences and a proper treatment
of how to train this term in a semi-supervised setting is important for this class of models.}}.

We next explain the evidence formulation for the full graphical model; and then introduce the module training
and joint training steps.

\noindent\textbf{Module and Joint Training.}
For the full model (including the answers $a$), the evidence is
$\mathcal{L} + \mathcal{U}_f$, where $\mathcal{U}_f=\sum_{m=1}^{M} \log p(\vx^m, a^m|\vi^m)$
and $\mathcal{L}=\sum_{n=1}^N \log p(\vx^n, \vz^n)$. Similar to the previous
section, one can derive a variational lower bound~\citep{Jordan1999VI}
on $\mathcal{U}_f$ (c.f.~\cite{Vedantam2018generative,Suzuki2017}):
\vspace{-10pt}
\begin{multline}\label{eqn:full_evidence_lower_bound}
    \mathcal{U}_f \ge
\sum_{m=1}^M E_{\vz \sim q_{\phi}(\vz | \vx^m)}
[\log p(a^m| \vi^m; \theta_{\vz}) + \\\log p_{\sigma}(\vx^m| \vz)] - \KLpq{q_{\phi}(\vz| \vx^m)}{\log p(\vz)}
\end{multline}\vspace{-10pt}

\textbf{Module Training.}
During module training, first we optimize the model only w.r.t the parameters responsible for
neural execution of the symbolic programs, namely $\theta_{\mathcal{Z}}$.
Concretely, we maximize:
\vspace{-5pt}
\begin{equation}
    \label{eqn:nmn_training}
    \max_{\theta_{\mathcal{Z}}} \sum_{m=1}^M E_{z\sim q(\vz| \vx^m)}[\log p_{\theta_z}(a^m|\vz, \vi^m)]
\end{equation}
The goal is to find a good initialization of the module parameters, say
$\theta_{\texttt{find[green]}}$ that binds the execution to the computations
expected for the symbol $\texttt{find[green]}$ (namely the neural module network).

\noindent\textbf{Joint Training.}
Having trained the question code and the neural module
network parameters, we train all terms jointly, optimizing the complete evidence
with the lower bound $\mathcal{L} + \mathcal{U}_f$.
We make changes to the above objective (across all the stages), by adding in
scaling factors $\alpha$, $\beta$ and $\gamma$ for corresponding terms in the
objective, and write out the $\KL$ term (\cref{eqn:full_evidence_lower_bound}),
subsuming it into the expectation:
\vspace{-10pt}
\begin{multline}\label{eqn:final_joint_elbo}
    \mathcal{L} + \mathcal{U}_f \approx
\sum_{m=1}^M E_{\vz \sim q_{\phi}(\vz | \vx^m)}
[\gamma \log p(a^m| \vi^m;\theta_{\vz}) + \\\log p_{\sigma}(\vx^m| \vz)- \beta \log q_{\phi}(\vz| \vx^m) + \beta \log p(\vz)] + \\ \sum_{n=1}^{N}\left[\alpha \log q_{\phi}(\vz^n| \vx^n) + \log p_{\sigma}(\vx^n|\vz^n) \right]
\end{multline}
\ramaarxiv{where $\gamma > 1$ is a scaling factor on the answer likelihood, which has fewer bits
of information than the question. For answers $a$, which have probability
mass functions, $\gamma > 1$ still gives us a valid lower bound}\footnote{Similar scaling
factors have been found to be useful in prior work~\citep{Vedantam2018generative} (Appendix A.3)
in terms of shaping the latent space.}. The same values of $\beta$ and $\alpha$ are used
as in question coding (and for the same reasons \ramaarxiv{explained above}).

\ramaarxiv{The first term, with expectation over $\vz\sim q_\phi(\cdot)$
is not differentiable with respect to $\phi$. Thus we use the} REINFORCE~\citep{Williams1992}
estimator with a moving
average baseline \ramaarxiv{to get a gradient estimate
for $\phi$ (see Appendix for more details.)}.
\ramaarxiv{We take the gradients (where available) for updating the
rest of the parameters.}

\subsection{Benefits of Three Stage Training}
\setlength{\textfloatsep}{5pt}
\begin{algorithm}[tb]
    \caption{\vqabayes{} Training Procedure}
    \label{alg:example}
 \begin{algorithmic}
    \STATE {\bfseries Given:} $\mathcal{D} = \{\vx^n, \vz^n\}_{n=1}^{N} \cup \{\vx^m, a^m, \vi^m\}_{m=1}^M$, $p(\vz)$
    \STATE {\bfseries Initialize:} $\theta_\mathcal{Z}$, $\sigma$, $\phi$\\
    \STATE {\bfseries Set:} \ramaarxiv{$\beta<1$, $\alpha > 1$, $\gamma > 1$}\\
    {\bfseries Question Coding}
    \STATE Estimate $\sigma$, $\phi$ optimizing~\cref{eqn:qc_elbo} $+$~\cref{eqn:approx_sup_evidence}\\
    {\bfseries Module Training}
    \STATE Estimate $\theta_{\mathcal{Z}}$ optimizing~\cref{eqn:nmn_training}\\
    {\bfseries Joint Training}
    \STATE Estimate $\theta_{\mathcal{Z}}$, $\sigma$, $\phi$ \ramaarxiv{optimizing~\cref{eqn:final_joint_elbo}}
 \end{algorithmic}
\textbf{Note:} Updates to $\phi$ from~\cref{eqn:qc_elbo} and~\cref{eqn:full_evidence_lower_bound}
 happen via. score function estimator, and the path derivative.
 Updates to $\theta_{\mathcal{Z}}$, $\sigma$ are computed using the gradient.
\end{algorithm}
\label{sec:discussionStageTraining}
In this section, we outline the difficulties that arise when we try to optimize
~\cref{eqn:final_joint_elbo} directly, without following the three stage procedure.
Let us consider question coding -- if we do not do question coding independently
of the answer, learning the parameters $\theta_z$ of the neural module network becomes
difficult, especially when $N << M$ as the mapping $p(a|\vz, \vi)$
is implemented using neural modules $p_{\theta_z}(a|\vi)$. This optimization
is discrete in the program choices, which hurts when $q_\phi$ is
uncertain (or has not converged). Next, training the joint model without
first running module training is possible, but trickier, because
the gradient from an untrained neural module network would pass
into the $q_\phi(\vz|\vx)$ inference network, adding noise to the updates.
Indeed, we find that inference often deteriorates when trained with
REINFORCE on a reward computed from an untrained network (\cref{tab:result}).

%% file: sections/related.tex
We first explain differences with other discrete structured latent variable
models proposed in the literature. We then connect our work to the broader
context of research in reasoning and visual question answering (VQA)
and conclude by discussing interpretability in the context of VQA.

\noindent\textbf{Discrete Structured Latent Variables.} 
\ramacr{Previous works have applied amortized variational inference~\citep{Kingma2013-wx,Rezende2014}
to build discrete, structured latent variable models, where the observations as well as the latents are either
sequences~\cite{Miao2016} or tree structured programs~\citep{Yin2018}.} While~\citet{Yin2018} consider the
problem of parsing text into programs,~\citet{Miao2016}
generate (textual) summaries using a latent variable~\cite{Miao2016}. In contrast,
\ramacr{our joint model
is richer since in addition to parsing, we also learn to decode a latent program
into answers by executing neural modules $\theta_\vz$.} Finally,
as discussed in~\cref{sec:vqabayes}, our derivation for the lower bound
on the question-program evidence provides some understanding of the objectives used
in these prior works for semi-supervised learning~\cite{Yin2018,Miao2016}.

\noindent\textbf{Visual Question Answering and Reasoning.} A number of approaches have studied
visual question answering, motivated to study multi-hop reasoning
~\citep{Hu2017NMN,Hu2018,Hudson2018,Johnson2017,Perez2017-gn,Santoro2017,Yi2018}. Some of these
works build in implicit, non-symbolic inductive biases to support compositional
reasoning into the network~\citep{Hudson2018,Hu2018,Perez2017-gn}, while others
take a more explicit symbolic approach~\citep{Yi2018,Hu2017NMN,Johnson2017,Mascharka2018,Yi2018}.
Our high level goal is centered around providing legible explanations and reasoning
traces for programs, and thus, we adopt a symbolic approach. Even in the realm of symbols,
different approaches utilize different kind of inductive biases in the mapping from
symbols (programs) to answer. While~\citet{Yi2018} favor an approach that represents
objects in a scene with a vectorized representation, and compute various operations
as manipulations of the vectors, other works take a more modular approach
~\cite{Andreas2016NMN,Hu2018,Johnson2017,Mascharka2018} where a set of symbols $\{z\}$
intsantiate neural networks with parameters $\theta_{\{z\}}$. We study the latter
approach since it is arguably more general and could conceivably transfer better 
to other tasks such as planning and control~\cite{Das2018}, lifelong learning
~\cite{Gaunt2017s,Valkov2018}~\etc

Different from all these
prior works, we provide a probabilistic scaffolding that embeds previous neural-symbolic
models, which we conceptualize should lead to better data-efficient legibility and 
the ability to debug coherence and sensitivity in reasoning. \ramapass{
We are not aware of any prior work on VQA where it is possible to reason about
coherence or sensitivity of the reasoning performed by the model.}

\noindent\textbf{Interpretable VQA.} Given its importance as a scene understanding task, and as 
a general benchmark for reasoning, there has been a lot of work in trying to interpret
VQA systems and explain their decisions~\citep{Das2016Att,Lu2016HieCoAtt,Park2018,Selvaraju2016}.
Interpretability approaches typically either perform some kind of explicit attention~\citep{Bahdanau2014}
over the question or the image~\citep{Lu2016HieCoAtt} to explain with a heat map the regions or parts
of the question the model used to arrive at an answer. Some other works develop post-hoc attribution 
techniques~\citep{Mudrakarta2018,Selvaraju2016} for providing explanations.
In this work, we are interested in an orthogonal
notion of interpretability, in terms of the legibility of the reasoning process used by the network given
symbolic instructions for a subset of examples.
More similar to our high-level motivation are approaches which take
a~\nsym{} approach, providing explanations in terms of programs used for reasoning about a question
~\cite{Andreas2016NMN,Hu2018}, optionally including a spatial attention over the image to localize the
function of the modules~\cite{Andreas2016NMN,Hu2017NMN,Mascharka2018}. In this work we augment the legibility
of the programs/reasoning from these approaches. 

%% file: sections/experiments.tex
\noindent\textbf{Dataset.} We report our results on the \ramaarxiv{CLEVR~\cite{Johnson2017} dataset}
and the SHAPES datasets
~\citep{Andreas2016NMN}. \ramaarxiv{The CLEVR dataset has been extensively used as a benchmark for
testing reasoning in VQA models in various prior
works~\cite{Hu2017NMN,Hu2018,Hudson2018,Johnson2017,Perez2017-gn,Santoro2017}
and is composed of 70,000 images and around 700K questions, answers and functional programs in the training set,
and 15,000 images and 150K questions in the validation set.
We choose first 20K examples from CLEVR v1.0 validation set and
use it as our val set. We report results on CLEVR v1.0 test set using the CLEVR evaluation
server on EvalAI~\citep{EvalAI}. The longest questions in the dataset are of
length 44 while the longest programs are of length 25. The question vocabulary has 89 tokens
while the program vocabulary has 40 tokens, with 28 possible answers.}

We investigate our design choices on the smaller
SHAPES dataset proposed in previous works~\citep{Andreas2016NMN,Hu2017NMN} for
visual question answering. The dataset is explicitly designed to test
for compositional reasoning, and contains compositionally novel questions
that the model must demonstrate generalization to
at test time. Overall there are 244 unique
questions with \texttt{yes}/\texttt{no} answers and 15,616 images~\citep{Andreas2016NMN}.
The dataset
also has annotated programs for each of the questions.
We use train, val, and test splits of 13,568, 1,024, and 1,024
$(\vx, \vz, \vi, \va)$ triplets respectively.
%
The longest questions in the dataset are of length 11 and shortest
are of length 4, while the longest programs are of length 6 and shortest
programs are of length 4. The size of the question vocabulary $\mathcal{X}$ is
14 and the program vocabulary $\mathcal{Z}$ is 12.

\noindent\textbf{Training.} On SHAPES, to simulate a data-sparse regime,
we restrict the set of question-aligned programs to
5, 10, 15, or 20\% of unique questions -- such that even
at the highest level of supervision, programs for
80\% of unique questions have never been seen during training. 
We train our program prior using a set of 1848 unique
programs simulated from the syntax (more details of this can be seen in the
Appendix).

In general, we find that performance at a given amount of supervision can be quite
variable. In addition, we also find
question coding, and module training stages tend to show
a fair amount of variance across multiple runs. To make fair comparisons, for every
experiment, we run question coding across 5 different runs, pick the best performing
model, and then run module training (updating $\theta_{\mathcal{Z}}$) across 10 different
runs. Next, we run the best model from this stage for joint training (sweeping across
values of $\gamma \in \{1, 10, 100\}$). Finally, at the end of this process we have the best
model for an entire training run. We repeat this process across five random 
datasets and report mean and variance at a given level of supervision.

\ramaarxiv{With CLEVR, we report results when we train on 1000 question-program supervision
pairs (this is $0.143$\% of all question-program pairs), along with the rest of the question,
image answer pairs (dropping the corresponding
programs). Similar to SHAPES, we report our results across 20 different choices of the subset
of 1000 questions to estimate statistical significance. In general, the CLEVR dataset
has question lengths with really large variance, and we select the subset of 1000 questions
with length at most 40. We do this to stabilize training (for both our
method as well as the baseline) and also to simulate a realistic scenario where an end
user would not want to annotate really long programs. For each choice of a subset of
1000 questions, we run our entire training pipeline (across the question coding,
module training and joint training stages) and report results on the top 15 runs
out of the 20 (for all comparisons reported in the paper), based on question coding
accuracy (see metrics below).}

\noindent\textbf{Metrics.} For question coding, we report the accuracy of the programs
predicted by the $q_\phi(\vz| \vx)$ model (determined with exact string match),
since we are interested in the legibility
of the generated programs. In module training, we report the VQA accuracy obtained by the
model, and finally, in joint training, we report both, VQA  and program
prediction accuracy since we are interested in both, getting the right answers and the legibility of the model's reasoning trace.

\ramacr{\noindent\textbf{Prior.} On SHAPES, we train the program prior $p(\vz)$ using programs simulated
using syntax (more details in the Appendix). On CLEVR, we train using 
all the programs in the training set, while restricting the number of paired
questions and programs to 1000 (as explained above).}

%
%
\noindent\textbf{Baselines.}
We compare against adaptation of a state-of-the-art semi-supervised learning approach
proposed for neural module networks by~\citet{Johnson2017}.~\citeauthor{Johnson2017} fit
the terms corresponding to $q_\phi(\vz| \vx)$ and $p(a|\vi; \theta_{\vz})$
in our model. Specifically, in question coding, they optimize
$\max_{\phi} \sum_{n} \log q_{\phi}(\vx^n|\vz^n)$ (where $n$ indexes data points
with associated questions, \ie the approach does not make use of ``unlabelled'' data).
Next, in module training 
they optimize
$\max_{\theta_{\mathcal{Z}}} \sum_{m=1}^M E_{z\sim q(\vz| \vx^m)}\left[\log p(a^m| \vi^m; \theta_{\vz})\right]$.
In joint training, they optimize the same objective with respect to the parameters
$\theta_{\mathcal{Z}}$ as well as $\phi$, using the REINFORCE gradient estimator:
$\nabla_{\phi} = \frac{1}{M}\sum_m [\log p(a| \vi; \theta_{\vz})-B] \nabla_{\phi} \log q_\phi(\vz| \vx^m)$,
where B is a baseline.
In contrast, we follow Algorithm 1 and maximize the corresponding evidence at every
stage of training.
We also report other baselines which are either ablations of our model or deterministic
variants of our full model. See Appendix for the exact setting of the hyperparameters.


%% file: sections/results.tex
\noindent\textbf{SHAPES results.}
We focus on evaluating approaches by varying
the fraction of question-aligned programs
(which we denote $\%\vx\leftrightarrow\vz$) in~\cref{tab:result}. 
To put the numbers in context, a baseline LSTM + image model, which
does not use module networks, gets  an accuracy of 63.0\% 
on Test (see~\citet{Andreas2016NMN}; Table~2). This indicates that SHAPES
has highly compositional questions which is challenging to model
via conventional methods for Visual Question Answering~\citep{Antol2015}.
Overall, we make the following observations:
\begin{compactitem}[\hspace{2pt}-]\vspace{-4pt}
\item \textbf{Our \vqabayes~approach consistently improves performance in data-sparse regimes.} While both methods tend to improve with greater program supervision,
\vqabayes~quickly outpaces NMN~\citep{Johnson2017}, achieving test VQA accuracies 30-35\% points higher for ${>}5\%$ program supervision.
Notably, both methods perform similarly poorly on the test set given 5\% program supervision, suggesting this may be too few examples to learn compositional reasoning.
\ramapass{Similarly, we see that the program prediction accuracy is also significantly higher
for our approach at the end of joint training,
meaning that~\vqabayes{} is right for the right reasons.}

\item \textbf{Our question coding stage greatly improves initial program prediction.}
Our \vqabayes{} approach to question coding gets approximately
double the question coding (program prediction) accuracy
as the NMN approach (col 1, question coding). This means that it effectively
propagates groundings from question-aligned
 programs during the coding phase. Thus the initial programs produced by our approach are more legible at a lower amount of supervision
 than NMN.
 Further, we also see improved VQA performance
 after the module and joint training stages which are
 based on predicted programs.

\begin{figure}
    \centering
    \includegraphics[width=1\columnwidth]{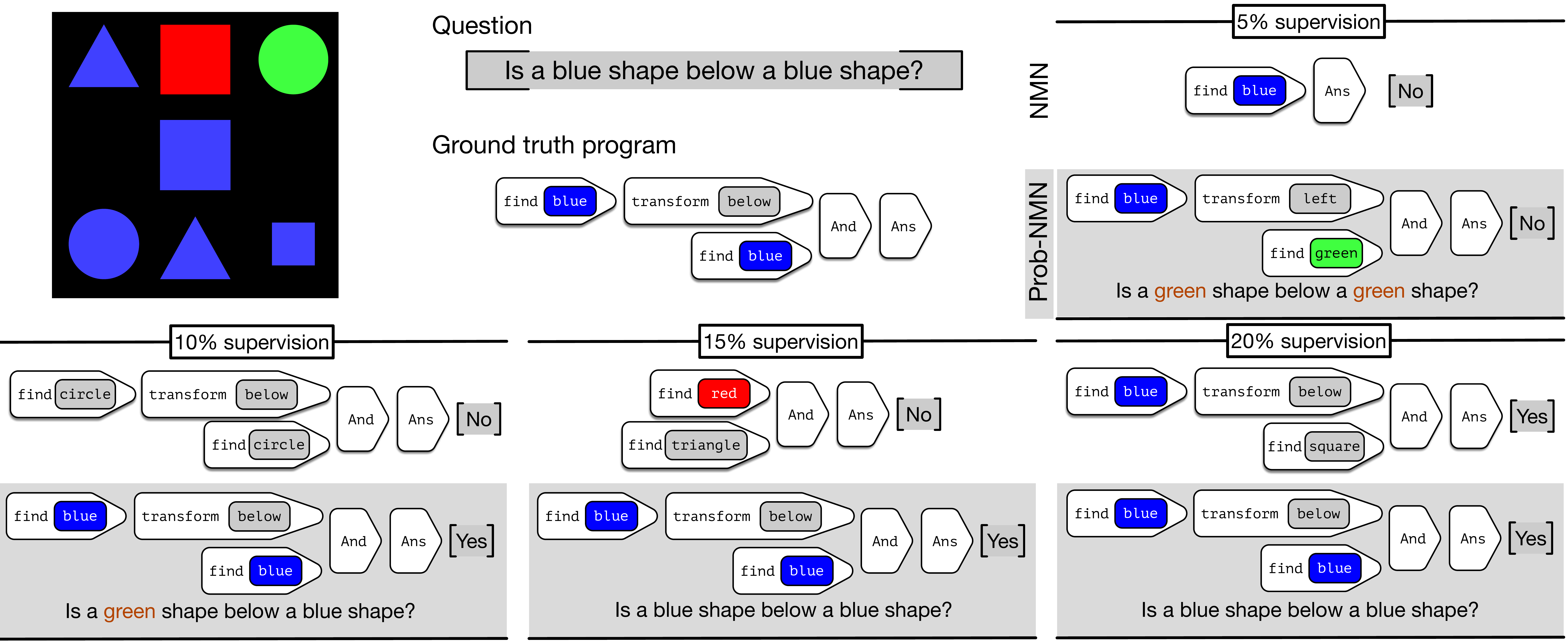}
    \vspace{-15pt}
    \caption{\textbf{Data-efficient Legibility:} Image with the corresponding question, ground truth
        program, and ground truth answer (top left). Predicted programs and answers of
        different models (NMN in white,~\vqabayes{} in gray) and reconstructed question (\vqabayes{}) for variable amount of supervision.
        We find~\vqabayes{} finds
        the right answer as well as the program more often than NMN~\citep{Johnson2017}.}
    \label{fig:example}
    \vspace{-5pt}
    \end{figure}
\item \textbf{Successful joint training improves program prediction.}
\ramapass{In general, we find that the accuracies obtained on program prediction deteriorate when the module training
stage is weak (row 1). On the other hand, higher program prediction accuracies generally lead to better module training,
which further improves the program prediction performance.}
\end{compactitem} \vspace{-4pt}
 

Figure \ref{fig:example} shows sample programs for each model. We see the limited supervision
negatively affects NMN program prediction,
 with the 5\% model resorting to simple \texttt{Find[X]${\rightarrow}$Answer} structures.
Interestingly, we find that the mistakes made by the~\vqabayes{} model,
\eg{}, green in 5\% supervision (top-right) are also made when reconstructing
the question (also substituting green for blue). Further, when the token does get corrected to blue, the question also
eventually gets reconstructed (partially correctly (10\%) and then fully (15\%)),
and the program produces the correct answer.
This indicates that there is high fidelity between the learnt question space, the
answer space and the latent space.

\begin{table*}[t]
\centering
\caption{Results (in percentage) on the SHAPES dataset with varying amounts of question-program supervision ($\%\vx\leftrightarrow\vz$), $\beta=0.1$.}
\resizebox{0.95\textwidth}{!}{
	\centering
	\renewcommand{\arraystretch}{1}
    \setlength{\tabcolsep}{3pt}
	\begin{tabular}{l c p{0pt} c c c c c p{2pt} c }
    	\toprule
        & \multirow{3}{*}{$\%\vx\leftrightarrow\vz$}
        & & \multicolumn{5}{c}{Validation During Training Stages} & & \multirow{2}{*}{\shortstack[c]{\\Test\\Accuracy}}\\ \cline{4-8}\noalign{\smallskip}
		& & & \multicolumn{2}{c}{\footnotesize [I] Question Coding} &\footnotesize [II] Module Training & \multicolumn{2}{c}{\footnotesize [III] Joint Training} & & \\
        & & & \scriptsize (Reconstruction)& \scriptsize \shortstack[c]{(Program Prediction)} & \scriptsize (VQA Accuracy) & \scriptsize (Program Prediction) & \scriptsize (VQA Accuracy) & & \scriptsize (VQA Accuracy)  \\
        \midrule 
        NMN \tiny\citep{Johnson2017} & \multirow{2}{*}{5} & & - & ~~9.28\pmvar{1.91} & 61.56\pmvar{3.59}& 0.0 \pmvar{0.0} & 63.08\pmvar{0.78}&& 60.06\pmvar{3.88}\\
        \vqabayes \tiny{(Ours)} &  & & 62.86\pmvar{5.31} & 17.12\pmvar{5.28} & 54.55\pmvar{11.31}& 28.12 \pmvar{28.12}& 72.50\pmvar{8.35} && 71.95\pmvar{11.15}\\[6pt]
        
        NMN \tiny\citep{Johnson2017} & \multirow{2}{*}{10} && - & 24.30\pmvar{2.39} & 60.31\pmvar{3.37}& 6.25 \pmvar{10.83} & 66.51\pmvar{4.10}&& 61.99\pmvar{0.96}\\
        \vqabayes \tiny{(Ours)} &  && 83.60\pmvar{5.57}& 60.18\pmvar{9.56} & 75.80\pmvar{3.62}& 90.62\pmvar{6.98}& ~96.86\pmvar{2.48}&& ~94.53\pmvar{2.06}\\[6pt]
        
        NMN \tiny\citep{Johnson2017} & \multirow{2}{*}{15} && - & 47.67\pmvar{5.02} & 69.47\pmvar{9.87}& 0.0 \pmvar{0.0} & 62.43\pmvar{0.49}&& 61.32\pmvar{2.36}\\
        \vqabayes \tiny{(Ours)} & & & 95.86\pmvar{0.20}& 84.85\pmvar{6.25} & 90.57\pmvar{3.44}&95.31 \pmvar{5.18} &98.40\pmvar{1.63}&& 97.02\pmvar{0.84}\\[6pt]
        
        NMN \tiny\citep{Johnson2017} & \multirow{2}{*}{20} && - & 58.37\pmvar{3.30} & 66.17\pmvar{7.02}& 43.75 \pmvar{43.75} & 80.68\pmvar{18.00}&& ~78.59\pmvar{19.27}\\
        \vqabayes \tiny{(Ours)} & && 96.10\pmvar{0.27}& 90.22\pmvar{1.63} & 91.81\pmvar{1.58}&96.87 \pmvar{5.41} &99.43\pmvar{0.61}&& 96.97\pmvar{1.30}\\
        \bottomrule
	\end{tabular}
}
\label{tab:result}
\end{table*}

\begin{table*}[t]
\centering
\caption{Results (in percentage) on the CLEVR dataset with $0.143$ \% $\vx\leftrightarrow\vz$, $\beta=0.1$. Validation metrics are calculated on a set of 20K examples out of CLEVR v1.0 validation split, across 15 random seeds. Test metrics are calculated on CLEVR v1.0 test split, and correspond to the best performing checkpoint across 15 random seeds.} \resizebox{0.95\textwidth}{!}{
	\centering
	\renewcommand{\arraystretch}{1}
    \setlength{\tabcolsep}{3pt}
	\begin{tabular}{l c p{0pt} c c c c c p{2pt} c }
    	\toprule
        & \multirow{3}{*}{$\%\vx\leftrightarrow\vz$}
        & & \multicolumn{5}{c}{Validation During Training Stages} & & \multirow{2}{*}{\shortstack[c]{\\Test\\Accuracy}}\\ \cline{4-8}\noalign{\smallskip}
		& & & \multicolumn{2}{c}{\footnotesize [I] Question Coding} &\footnotesize [II] Module Training & \multicolumn{2}{c}{\footnotesize [III] Joint Training} & & \\
        & & & \scriptsize (Reconstruction)& \scriptsize \shortstack[c]{(Program Prediction)} & \scriptsize (VQA Accuracy) & \scriptsize (Program Prediction) & \scriptsize (VQA Accuracy) & & \scriptsize (VQA Accuracy)  \\
        \midrule 
        NMN \tiny\citep{Johnson2017} & \multirow{2}{*}{0.143} & & - & 62.47\pmvar{9.82} & 79.26\pmvar{4.03}& 63.08 \pmvar{9.91} & 79.38\pmvar{4.21}&& 75.71\\
        \vqabayes \tiny{(Ours)} &  & &  - & 93.15\pmvar{8.61} & 94.42\pmvar{3.77}& 93.87 \pmvar{8.73}& 95.52\pmvar{4.15} && 97.73\\[6pt]
        \bottomrule
	\end{tabular}
}
\label{tab:result_clevr}
\end{table*}


Finally, the N2NMN approach~\citep{Hu2017NMN} evaluates their question-attention based module networks in the fully unsupervised setting,
getting to 96.19\% on TEST. However, the programs in this case
become non-compositional (see~\cref{sec:related_work}), as the model leaks information from questions to answers via attention,
meaning that programs no longer carry the burden of explaining the observed answers. This makes the modules illegible. In general,
our approach makes the right independence assumptions ($a\perp \vx | \vi, \vz$) which helps legibility to emerge, along with our
careful design of the three stage optimization procedure.

\ramacr{In the Appendix, we show additional
comparisons to a deterministic variant of our model (with $\beta=0$), and study the effect of optimizing the true ELBO.}

\noindent\textbf{CLEVR Results.} Our results on the CLEVR dataset~\cite{Johnson2017} reflect similar trends as the results 
on SHAPES (\cref{tab:result_clevr}). As explained in~\cref{sec:experiments}, we work with 1000 supervised question
program examples from the CLEVR dataset (0.143\% of all question program pairs).
With this, at the end of question coding,~\vqabayes{} gets to an accuracy of 93.15 $\pm$ 8.61,
while the baseline NMN approach gets to an accuracy of 62.47 $\pm$ 9.82.
These gains for the~\vqabayes{} model are reflected
in module training, where the~\vqabayes{} approach gets to an accuracy of 94.42 $\pm$ 3.77,
while the baseline is at 79.26 $\pm$ 4.03. At the end of joint training these improve
marginally to 95.52 $\pm$ 4.15 and 79.38 $\pm$ 4.21 respectively. Crucially, this is achieved
with a program generation accuracy of 93.87 $\pm$ 8.73 by our approach compared to a baseline
accuracy of 63.08 $\pm$ 9.91. Thus, the programs generated by the~\vqabayes{} model are more
legible than those by the semi-supervised NMN approach~\cite{Johnson2017}. Finally, the same
trends are reflected on the CLEVR test set as well, with \vqabayes{} outperforming NMN.
We refer to the Appendix for more details of the training regime and architectural choices
\footnote{We find the question reconstruction accuracy is close to zero, since
the model often decodes to a semantically equivalent question string which an exact match
 does not recover, thus omit that number from~\cref{tab:result_clevr}."}.

\noindent\textbf{Coherence and Sensitivity in Reasoning.}
Next we show that the probabilistic formulation can be used to
check a model's coherence in reasoning (see~\cref{fig:coherence_sensitivity}, top). Given
the image and the answer \texttt{yes}, we observe that one is able to generate multiple,
diverse reasoning patterns which lead to the answer, by sampling $\vz\sim p(\vz|a, \vi)$,
showing a kind of systematicity in reasoning~\citep{Lake2017}. On the other hand,
when we change the answer to \texttt{no} (see~\cref{fig:coherence_sensitivity}, bottom),
keeping the image the same, we observe
that the reasoning pattern changes meaningfully, yielding a program that evaluates
to the desired answer. See Appendix for a description of how we do the sampling and results
on CLEVR.
\begin{figure}
   \vspace{-5pt}
    \centering
    \includegraphics[width=0.8\columnwidth]{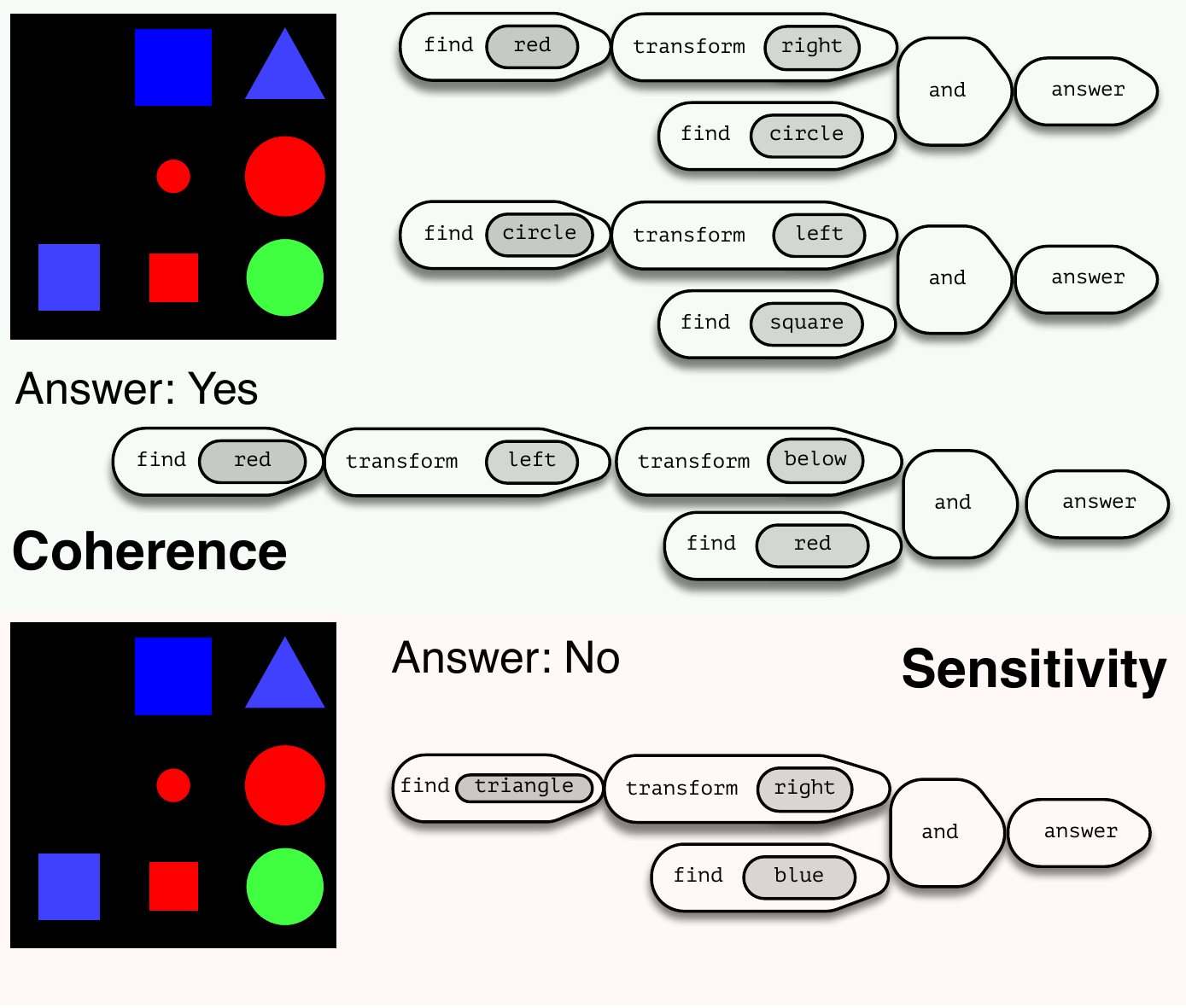}
    \vspace{-10pt}
    \caption{\textbf{Coherence and Sensitivity}: Image and a correponding answer are specified,
    and the model is asked to produce programs it believes should 
    lead to the particular answer for the given image. One can notice that the generated
    programs are consistent with each other, and evaluate to
    the specified answer. Results are shown for a \vqabayes{} model
    trained with 20\% program supervision on SHAPES.}
    \label{fig:coherence_sensitivity}
  \vspace{-5pt}
\end{figure}

%% file: sections/conclusion.tex
In this work, we discussed a probabilistic, sequential latent variable model for visual question answering, that jointly learns to parse
questions into programs, reasons about abstract programs,
and learns to execute them on images using modular neural networks.
We demonstrate that the probabilistic model endows the model with desirable properties
for interpretable reasoning systems, such as the reasoning being clearly legible given minimal number of teaching examples, and
the ability to probe into the reasoning patterns of the model, by testing their coherence (how consistent are the reasoning
patterns which lead to the same decision?) and sensitivity (how sensitive is the decision to the reasoning pattern?).
We test our model on the \ramaarxiv{CLEVR dataset as well as a dataset of compositional questions}
about SHAPES and find that handling stochasticity enables better generalization
to compositionally novel inputs.

%% file: appendix/appendix.tex
\section{A variational lower bound on the evidence for sequential latent variable models}
\begin{nonumlemma}
    Given observations $\{\vx^n, \vz^n\}$ and $p(\vx, \vz) = p(\vz) p_{\sigma}(\vx| \vz)$,
    let $z_t$, the token at the $t^{th}$ timestep in a sequence $\vz$ be distributed
    as a categorical with parameters $\pi_t$. Let us denote 
    $\Pi = \{\pi_t\}_{t=1}^{T}$, the joint random variable over all $\pi_t$. Then,
    the following is a lower bound on the joint evidence $\mathcal{L} = \sum_n \log p(\vz^n) + \log p_\sigma(\vx^n| \vz^n)$:
    \begin{multline}
    \mathcal{L} \ge \sum_{n=1}^N \log q_{\phi}(\vz^n|\vx^n) + \log p_{\sigma}(\vx^n|\vz^n) -\\
    \KLpq{q(\Pi|\vz^n, \vx^n)}{p(\Pi)}
    \end{multline}\label{eqn:question_code_sup}
    where $p(\Pi)$ is a distribution over the sampling distributions implied by the prior $p(\vz)$
    and $q(\Pi|\vx^n, \vz^n) = q(\pi_0) \prod_{t=1}^T q(\pi_t|\vx^n, z_0^n, \cdots, z_{t-1}^n)$, where
    each $q(\pi_t|\cdot)$ is a delta distribution on the probability simplex.
\end{nonumlemma}

\begin{proof}
We provide a proof for the lemma stated in the main paper, which explains
how to obtain a lower bound on the evidence $\mathcal{L}$ for the data
samples $\{\vx^n, \vz^n\}_{n=1}^{N}$.

Let us first focus on the prior on (serialized) programs, $\log p(\vz)$.
This is a sequence model, which factorizes as,
$\log p(\vz) = \log p(z_1) + \sum_{t=2}^T \log p(z_t|\{z_i\}_{i=1}^{t-1})$, where
$p(z_t|\cdot) \sim Cat(\pi_t)$. We learn point estimates of the parameters for the prior
(in a regular sequence model), however, the every sequence has an implicit distribution
on $p(\pi_t)$. To see this, note that at every timestep (other than $t=1$), the sampling
distribution depends upon $\{z_{i}\}_{i=1}^T$, which is set of random variables. 
Let $p(\Pi, \vz)$ denote the joint distribution of all the sampling distributions
across multiple timesteps $\Pi = (\pi_1,\cdots, \pi_T)$ and $\vz$, the realized
programs from the distribution. Let $\mathcal{G} = \{\pi_i, \vz_i\}_{i=1}^{T}$

Since $\pi_t$ refers to the sampling distribution for $z_t$, we immediately
notice that $z_t \perp \mathcal{G}/\pi_t| \pi_t$. That is, $z_t$ only
depends on the sampling distribution at time $t$, meaning that
$\log p(\vz|\Pi) = \sum_{t=1}^{T} \log p(z_t|\pi_t)$.

With this (stochastic) prior parameterization, we can factorize the full
graphical model now as follows: $p(\vx, \vz, \Pi) = p(\Pi) p(\vz|\Pi) p_\sigma(\vx|\vz)$.
We treat $\Pi$ as a latent
variable. Then, the marginal likelihood for this model is:

\begin{equation}
	\mathcal{L} = \sum_{n=1}^N \log p(\vz^n) + \log p_\sigma(\vx^n|\vz^n)
\end{equation}

Notice that this equals the evidence for the original, graphical model
of our interest. Next, we write, for a single data instance:
\begin{align*}
	\log p(\vx, \vz) &= \log \int d\Pi p(\vx, \vz, \Pi)\\
	&= \log \int d\Pi p(\Pi) p(\vz|\Pi) p_\sigma(\vx|\vz)\\
	&= \log \int d\Pi \frac{p(\Pi)}{q(\Pi|\vz, \vx)} p(\vz|\Pi) p_\sigma(\vx|\vz) q(\Pi|\vz, \vx)
\end{align*}
where we multiply and divided by a variational approximation $q(\Pi|\vz, \vx)$ and apply
Jensen's inequality (due to concavity of $\log$) to get the following variational lower bound:

\begin{multline}
	\log p(\vx, \vz)
    \ge \int d\Pi q(\Pi|\vz, \vx) \left[\log p(\vz|\Pi) + \log p_\sigma(\vx|\vz)\right]\\- \KLpq{q(\Pi|\vz,\vx)}{p(\Pi)} \label{elbo:var_approx}
\end{multline}

Next, we explain how to parametrize the prior, which is the key assumption
for deriving the result. Let $\hat{\pi}_t = f(\vx, \{\vz_i\}_{i=1}^{t-1})$ be the
inferred distribution for $\vz_t$. That is, $\vz_t = Cat(\hat{\pi_t})$. Then,
let $q_\phi(\pi_t|\vx, \{\vz_i\}_{i=1}^{t-1}) = \delta(\pi_t - \hat{\pi}_t)$, that is,
a Dirac delta function at $\hat{\pi}_t$.

Further, note that we can write the second term
$E_{q(\Pi|\vz, \vx)}\left[\log p(\vz|\Pi)\right]$ as follows:

\begin{align}\label{eqn:sum_written_out}
	&\int d\Pi q(\Pi|\vz, \vx) \log p(\vz|\Pi) \\
	&=\sum_{t=1}^T \int d\pi_t q(\pi_t|\{z_i\}_{i=1}^{t-1}, \vx) \log p(z_t|\pi_t)
\end{align}

Each integral in the sum above can be simplified as follows:
\begin{align*}
	&\int d\pi q(\pi_t|\{z_i\}_{i=1}^{t-1}, \vx) \log p(z_t|\pi_t) \\
	& = \int d\pi \delta(\pi_t - \hat{\pi}_t) \log p(z_t|\pi_t) \\
	& = \log p(z_t|\hat{\pi}_t)\\
	& = \log q_\phi(z_t|\{z_i\}_{i=1}^{t-1}, \vx) \texttt{   [by definition]}
\end{align*}

Substituting into~\cref{eqn:sum_written_out}, we get:
\begin{align*}
&\sum_{t=1}^T \int d\pi_t q(\pi_t|\{z_i\}_{i=1}^{t-1}, \vx) \log p_{\sigma}(z_t|\pi_t)\\
& = \sum_{t=1}^T \log q_\phi(z_t|\{z_i\}_{i=1}^{t-1}, \vx) \\
& = \log q_\phi(\vz|\vx)
\end{align*}

This gives us the final bound, written over all the observed data points:
\begin{multline}
	\mathcal{L} \ge \sum_{n=1}^N \log p_\sigma(\vx^n|\vz^n) + \log q_\phi(\vz^n|\vx^n) \\
	- \KLpq{q(\Pi|\vx, \vz)}{p(\Pi)}
\end{multline}
\end{proof}

\section{Justification for $\alpha$}
We optimize the full evidence lower bound $\mathcal{L}_{qc} + \mathcal{U}_{qc}$ using supervised learning
for $\mathcal{L}_{qc}$ and REINFORCE for $\mathcal{U}_{qc}$. 
When using REINFORCE, notice that $\nabla_{\phi} \mathcal{U}_{qc} = \nabla_{\phi} \log q(\vz| \vx^m) [R_{qc} - B]$,
where $B$ is a baseline.
For the lower bound on $\mathcal{L}_{qc}$, the gradient is $\nabla_{\phi} \log q(\vz^n|\vx^n)$.
Notice that the scales of the gradients from the two terms are very different, since the order
of $R_{qc}$ is around the number of bits in a program, making it dominate in the early stages of 
training.  

\section{Justifications for $\beta<1$}
During question coding, we learn a latent representation of programs in presence of a question reconstructor.
Our reconstructor is a sequence model, which can maximize the likelihood of reconstructed question without
learning meaningful latent representation of programs. Recent theory from~\citet{Alemi2017broken} prescribes
changes to the ELBO, setting $\beta<1$ as a recipe to learn representations which avoid learning degenerate
latent representations in the presence of powerful decoders.
\reb{Essentially,~\citet{Alemi2017broken} identify that up to a constant factor,
the negative log-marginal likelihood term $D$=$-E_{z \sim q_{\phi}(\vz | \vx)}
[\log p_{\sigma}(\vx| \vz)]$, and the KL divergence term $R=\KL( q_{\phi}(\vz| \vx), p(\vz))$
bound the mutual information between the data $\vx$ and the latent variable $\vz$ as follows:}
\begin{equation}
	\reb{H - D \leq I(\vx, \vz) \leq R}
\end{equation}
\reb{where $H$ is the entropy of the data, which is a constant. 
One can immediately notice that the standard $\elbo=-D-R$. This means that for the same
value of the $\elbo$ one can get different models which make drastically different usage of the
latent variable (based on the achieved values of $D$ and $R$). Thus, one way to achieve
a desired behavior (of say high mutual information $I(\vx, \vz)$), is to set $\beta$ to a lower
value than the standard $\elbo$ (c.f. Eqn. 6 in~\citet{Alemi2017broken}.)}
This aligns with our goal in question coding, and so we use $\beta<1$.

\section{Implementation Details for SHAPES}
We provide more details on the modeling choices and hyperparameters used to
optimize the models in the main paper.

\noindent\textbf{Sequence to Sequence Models:} \retwo{All our sequence to sequence models described in the main paper (Section. 2)
are based on LSTM cells with a hidden state of 128 units, single layer depth,
and have a word embedding of 32 dimensions for both the question as well
as the program vocabulary.}

The when sampling from a model, we sample till the maximum sequence length
of 15 for questions and 7 for programs.

\noindent\textbf{Image CNN:} The SHAPES images are of \texttt{30x30}
in size and are processed by a two layered convolutional neural network,
which in its first layer has a \texttt{10x10} filters applied at a stride
of 10, and in the second layer, it does \texttt{1x1} convolution applied
at a stride of 1. Channel widths for both the layers are 64 dimensional. 

\noindent\textbf{Moving average baseline:} For a reward $R$, we use an 
action independent baseline $b$ to reduce variance, which tracks the moving average of the 
rewards seen so far during training. Concretely, the form for
the $q(\vz|\vx)$ network is:
\begin{equation}
\nabla J(\theta) = \mathbf{E}_{\vz\sim q(\vz| \vx)}[(R - b_t) \nabla \log q(\vz| \vx)]
\end{equation}
where, given $D$ as the decay rate for the baseline,
\begin{equation}
b_t = b_{t-1} + D * (R - b_{t-1})
\end{equation}
is the update on the baseline performed at every step of training.

In addition, the variational parameters $\phi$ are also updated via. the path
derivative $\nabla_\phi \log q_\phi(\vz| \vx)$.


\noindent\textbf{Training Details:} We use the ADAM
optimizer with a learning rate of \texttt{1e-3}, a minibatch
size of 576, and use a moving average baseline for REINFORCE,
with a decay factor of 0.99. Typical values
of $\beta$ are set to 0.1, $\alpha$ is set to 100 and $\gamma$ is chosen
on a validation set among \reb{(1.0, 10.0, 100.0)}.

\noindent\textbf{Simulating programs from known syntax:} We follow a simple
two-stage heuristic procedure for generating a set of samples to train
the program prior in our model. Firstly, we build a list of possible
future tokens given a current token, and sample a random token from that
list, and repeat the procedure for the new token to get a large set of
possible valid sequences -- we then pass the set of candidate sampled
sequences through a second stage of filtering based on constraints
from the work of~\citep{Hu2017NMN}.

\section{Additional Results on SHAPES}
\noindent\textbf{Empirical \vs{} syntactic priors.}
\reb{While our default choice of the prior $p(\vz)$ is from programs simulated
from the known syntax, it might not be possible to exhaustively enumerate
all valid program strings in general. Here, we consider the special
case where we train the prior on a set of \emph{unaligned},
ground truth programs
from the dataset. Interestingly, we find that the performance of the question
coding stage, especially at reconstructing the questions improves significantly
when we have the syntactic prior as opposed to the empirical prior (from
38.39 $\pm$ 11.92\% to 54.86 $\pm$ 6.75\% for 5\% program supervision). In terms
of program prediction accuracy, we also observe marginal improvements in the
cases where we have 5\% and 10\% supervision respectively, from 56.23 $\pm$ 2.81\%
to 65.45 $\pm$ 11.88\%. When more supervision is available, regularizing with
respect to the broader, syntactic prior hurts performance marginally, which
makes sense, as one can just treat program supervision as a supervised
learning problem in this setting.}

\noindent\textbf{Predicate Factoring.} \ramacr{We choose to parameterize a \texttt{find[green]}
operation with its own individual parameters $\theta_{\texttt{find[green]}}$
instead of having a parameterization for $\theta_{\texttt{find}}$ with $\texttt{green}$
supplied as an argument to the neural module network. In doing so we followed prior
work~\cite{Johnson2017}. However, we also note that other work~\cite{Hu2017NMN} has
explored the second parameterization.
Conceptually, choice of either is orthogonal to the benefits from our
probabilistic formulation. However, for completeness we implemented argument
based modules on SHAPES -- getting similar accuracy as atomic modules
(for both Prob-NMN and NMN). For example, at 15\% supervision, \vqabayes{} gets to a 
joint training VQA accuracy on (SHAPES-Val) of 97.53 ($\pm$ 1.23) [Prob-NMN] \vs 80.78
($\pm$ 11.75) for NMN. Thus, this design choice does not seem to have an impact 
on the relative performance of~\vqabayes{} and the NMN baseline.}

\noindent\textbf{Effect of the Program Prior $\mathbf{p(\vz)}$.}
Next, we explore the impact of regularizing the program posterior to be close to the prior $p(\vz)$,
for different choices of the
prior. 
Firstly, we disable the KL-divergence term by setting $\beta=0$ in
Equation (6), recovering a deterministic version of
our model, that still learns to reconstruct the question $\vx$ given a sampled
program $\vz$. Compared to our full model at 10\% supervision,
the performance on question coding drops to
23.14 $\pm$ 6.1\% program prediction accuracy (from 60.18 $\pm$ 9.56\%).
Interestingly, the $\beta=0$ model has a better 
performance on question reconstruction, which improves from
83.60 $\pm$ 5.57 \% to 94.3 $\pm$ 1.85\%.
This seems to indicate that, without the $\KL$ term, the model focuses solely on reconstruction and fails to learn compositionality in
the latent $\vz$ space, such that supervised grounding are poorly propagated to unsupervised questions as evidenced by
the drop in program prediction accuracy. Thus, the probabilistic model helps better achieve our high-level goal of
data-efficient legibility in the reasoning process.
In terms of the VQA accuracy on validation (at the end of joint training)
we see a drop in performance from $96.86 \pm 2.48$ to $74.82\pm 17.88$.

\noindent\textbf{Effect of Optimizing the True ELBO, $\beta=1$.}
Next, we empirically validate the intuition that for the semi-supervised setting,
approximate inference $q_\phi(\vz| \vx)$
learned using Equation (1) for sequence models often leads to solutions that do not make
meaningful use of the latent variable. For example, we find that at $\beta=1$ (and 10\% supervision),
training the true
evidence lower bound on $\mathcal{U}_{qc}$ deteriorates the question reconstruction accuracy
from 83.60 \% to 8.74\% causing a large distortion in the original question, as predicted
by the results from~\citet{Alemi2017broken}. Overall, in this setting the final accuracy on VQA
at the end of joint training is $74.45\pm 15.58$, a drop of $>$ 20 \% accuracy.

\begin{figure*}[t]
  \vspace{-5pt}
   \centering
   \includegraphics[width=\textwidth]{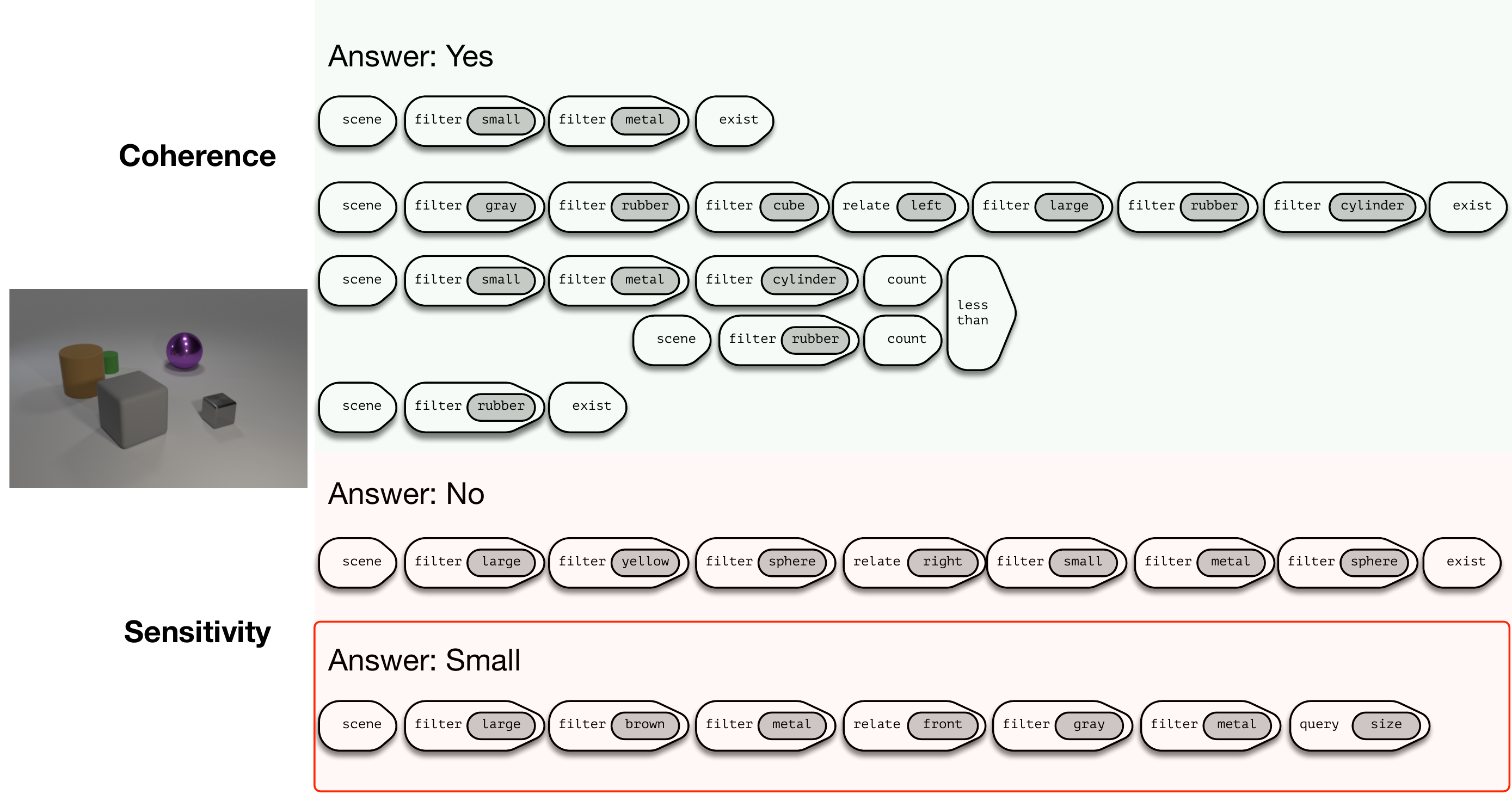}
   \vspace{-10pt}
   \caption{\textbf{Coherence and Sensitivity}: \ramacr{Image and a correponding answer are specified,
   and the model is asked to produce programs it believes should 
   lead to the particular answer for the given image. One can notice that the generated
   programs are consistent with each other, and evaluate to
   the specified answer. Results are shown for a \vqabayes{} model
   trained with 0.143\% supervision on CLEVR.}}
   \label{fig:coherence_sensitivity_clevr}
 \vspace{-10pt}
\end{figure*}

\section{Coherence and Sensitivity Sampling}
Our goal is to sample from the posterior $p(\vz|a, \vi)$, which we accomplish
by sampling from the unnormalized joint distribution
$p(\vz|a, \vi) \propto p(\vz) p(a|\vz, \vi)$. 
To answer a query conditioned on an answer $\hat{a}$,
one can simply filter out samples which have $a = \hat{a}$, and
produce the sample programs as the program, and decode the program $\vx^i \sim p(\vx|\vz^i)$
to produce the corresponding question. Note that it is also possible
to fit a (more scalable) variational
approximation $q_{\sigma}(\vz|a, \vi)$ having trained
the generative model retrospectively, following results from~\citet{Vedantam2018generative}
However, for the current purposes we found the above sampling procedure to be sufficient.

\section{Implementation Details on CLEVR}

CLEVR dataset is an order of magnitude larger than SHAPES, with larger vocabulary, longer sequences and images with higher visual complexity. We take this scale-up into account, and make suitable architectural modifications in comparison to SHAPES.

\noindent\textbf{Sequence to Sequence Models:} The encoders of our sequence-to-sequence models (program generator and question reconstructor) are based on LSTM cells with a hidden state of 256 units, a depth of two layers, and accept input word embeddings of 256 dimensions for both question and program vocabulary. We follow the same architecture for our program prior sequence model. Our decoders are similar to encoders, only with a modification of having single layer depth. When sampling from the model, we sample till the maximum sequence length in observed training data -- this results in program sequences of maximum length 28, and question sequences of maximum length 44.

\noindent\textbf{Image CNN:} We resize the CLEVR images to \texttt{224x224} dimensions and pass them through an ImageNet pre-trained ResNet-101, and obtain features from
its `third stage', which results in features of dimensions \texttt{1024x14x14}. These features are used as input to the Neural Module Network.

\noindent \textbf{Training Details} In order to get stable training with
sequence-to-sequence models on CLEVR, we found it useful to
average the log-probabilities across multiple timesteps. Without this,
in presence of a large variance in sentence lengths, we find that the model
focuses on generating the longer sequences, which is usually hard without
first understanding the structure in shorter sequences. Averaging (across) the
$T$ dimension helps with this. \ramaarxiv{This is also common practice when working with sequence
models and has been done in numerous prior works~\cite{Vinyals2015,Johnson2017,Hu2017NMN,
Lu2017}~\etc}. With higher linguistic variation in CLEVR questions, learning question reconstruction is difficult at the start of training when the latent program space isn't learned properly. To make learning easier, we increase the influence of teaching examples by also scaling question reconstruction loss with $\alpha$ in Equation (3). We use the ADAM
optimizer with a learning rate of \texttt{1e-3} for Question Coding, \texttt{1e-4} for Module Training and \texttt{1e-5} for Joint Training. We do not introduce any weight decay. We use a batch size of 256 across all stages, and a moving average baseline for REINFORCE with decay factor of 0.99, similar to SHAPES. We set $\alpha$ as 100, $\beta$ as 0.1 and $\gamma$ as 10. \ramaarxiv{We validate every 500 iterations for question coding and joint training, and every 2000 iterations for module training. We drop the learning rate to half when our primary metric (either program prediction accuracy or VQA accuracy, in respective stages) does not improve within next 3 validations. We use the code and module design choices of the state of the art TbD nets~\cite{Mascharka2018} to implement the $p_{\theta_\vz}(a| \vi)$ mapping in our model.}

\section{Results on CLEVR-Humans dataset}
\ramacr{\citet{Johnson2017} also collect a dataset of questions asked by humans
on the images in the CLEVR dataset, which is called CLEVR-humans. CLEVR-humans
is out of distribution data for both NMN and Prob-NMN. We train the models
on the regular CLEVR dataset (with 0.143~\% question-program supervision)
and evaluate the checkpoints on the CLEVR-human validation set which contains
7202 questions. We find that the gains on the regular CLEVR setup for
Prob-NMN also translate to the CLEVR-humans setting, with Prob-NMN getting
to an accuracy of 50.81 $\pm$ 1.93 \vs 45.45 \% for NMN. Both approaches
perform better than chance accuracy of 15.66\% computed by picking
the most frequent answer from the training set.
We compute $p(a|\vx, \vi) \approx \frac{1}{N} \sum_{n} p(a|\vz^n, \vi)$,
where $\vz^n \sim q(\vz|\vx)$ for this analysis, drawing $n=20$ samples.
Next we were interested in understanding if modeling the uncertainty via. Prob-NMN
helped more with out of distribution questions (such as CLEVR-humans) --
this phenomena is also called aleatoric uncertaintly. For
this, we varied the number of samples we drew, sweeping from $n=1$ to $n=20$.
For a single sample, we found that the performance drop for both NMN
as well as Prob-NMN was somewhat similar, \ie the NMN performance dropped
to 44.98 $\pm$ 2.67 while Prob-NMN dropped to 50.26 $\pm$ 2.01.
In general, we do not see a greater benefit of handling aleatoric uncertaintly via.
averaging for Prob-NMN over NMN when tested in the out-of-domain CLEVR-humans
dataset. Our hypothesis for this is that amortized inference is not handling
aleatoric uncertainty as well as we desire in our application.}

\section{Coherence and Sensitivity results on CLEVR}
\ramacr{We repeat the same procedure used for the SHAPES dataset on the CLEVR
dataset to analyze the coherence and sensitivity in the reasoning
performed by the \vqabayes{} model. Given an image and a candidate answer,
say ``yes'' as input, we find that multiple sample programs drawn from
the model are coherent with respect to each other. In~\cref{fig:coherence_sensitivity_clevr}
we show a sorted list of such programs ranked by log-probability of programs, $\vz$. Next,
when we switch the answer to ``no'', we find that the program changes in a
sensible manner, yeilding something that should correctly evaluate to the specified
answer. However, in general we find that the results are not as interpretable
for non ``yes''/``no'' answers. For example, when conditioning on ``small'', we find
that the model makes a mistake of asking for the object in front of ``large brown
metal'' object in stead of ``large brown'' object. This might be because
the proposed model does not take into account question premise (i.e. does not 
have an arrow from $\vi$ to $\vz$) meaning that a question about a large brown
metal object would never be asked about this image, since it is not present in
it.}